\documentclass{wlscirep}

\usepackage{amsmath,amssymb} 
\usepackage{amsthm}

\newtheorem{prop}{Proposition}

\newtheorem{claim}{Claim}

\DeclareMathOperator*{\argmax}{\mathrm{arg~max}}

\usepackage{color}

\title{Guaranteed satisficing and finite regret: Analysis of a cognitive satisficing value function}

\author[1]{Akihiro Tamatsukuri}
\author[2,3,*]{Tatsuji Takahashi}
\affil[1]{Graduate School of Advanced Science and Engineering, Tokyo Denki University, Ishizaka, Hatoyama, Hiki, Saitama 350-0394, Japan}
\affil[2]{School of Science and Engineering, Tokyo Denki University, Ishizaka, Hatoyama, Hiki, Saitama 350-0394, Japan}
\affil[3]{Dwango Artificial Intelligence Laboratory, 
5-24-5 Hongo, Bunkyo, Tokyo 113-0033, Japan}

\affil[*]{tatsujit@mail.dendai.ac.jp}

\begin{abstract}
As reinforcement learning algorithms are being applied to increasingly complicated and realistic tasks, it is becoming increasingly difficult to solve such problems within a practical time frame. 
Hence, we focus on a \textit{satisficing} strategy that looks for an action whose value is above the aspiration level (analogous to the break-even point), rather than the optimal action. 
In this paper, we introduce a simple mathematical model called risk-sensitive satisficing ($RS$) that implements a satisficing strategy by integrating risk-averse and risk-prone attitudes under the greedy policy. 
We apply the proposed model to the $K$-armed bandit problems, which constitute the most basic class of reinforcement learning tasks, and prove two propositions.
The first is that $RS$ is guaranteed to find an action whose value is above the aspiration level.
The second is that the regret (expected loss) of $RS$ is upper bounded by a finite value, given that the aspiration level is set to an ``optimal level'' so that satisficing implies optimizing. 
We confirm the results through numerical simulations and compare the performance of $RS$ with that of other representative algorithms for the $K$-armed bandit problems. 
\end{abstract}
\begin{document}

\flushbottom
\maketitle
% ^ <tatsuji.takahashi@gmail.com> 2018-05-22T11:36:51.185Z.
\thispagestyle{empty}

\section*{Introduction}

Reinforcement learning (RL), a framework for learning and control in which agents search for proper actions in an environment through trial and error, 
has witnessed rapid development in recent years, as evidenced by the super-human performances of deep Q-networks (DQN)~\cite{Mnih2015} in video game playing and AlphaGo~\cite{Silver2016} in the game of Go.
Moreover, the application range of RL extends not only to more complicated tasks on computers but also to the control of robots~\cite{Muse2009} and unmanned aerial vehicles (UAVs)~\cite{Zhao2018} in the real world.

As RL algorithms are being applied to increasingly complicated and realistic tasks, the limits of sensors, processors, and actuators of agents are posing serious obstacles for conventional optimization algorithms. 
Simon proposed the notion of \textit{bounded rationality} as the principle underlying agents' behavior under resource limits~\cite{Simon1957}.
A bounded rational agent may appear to behave irrationally, but by considering the limits and constraints, the agent's behavior can be understood as rational.
Bounded rationality has attracted considerable attention in recent years. Computational rationality~\cite{Lewis2014}, which has been claimed to integrate the three fields of neuroscience (brain), cognitive science (mind), and artificial intelligence (machine)~\cite{Gershman2015}, is an updated form of bounded rationality. 
Further, it has been proposed that abstraction and hierarchy, which have been considered to enable flexible and efficient cognition of humans~\cite{Tenenbaum2011}, result from the above-mentioned limitations and are bounded rational~\cite{Genewein2015}. 

The representative decision making policy in the theory of bounded rationality is \textit{satisficing}~\cite{Simon1955, Simon1956}.
Satisficing agents do not keep searching for the optimal action; instead, they stop searching when an action whose quality is above a certain level (aspiration) is found. 
The satisficing strategy has not attracted much attention in reinforcement learning, except for a few studies~\cite{Bendor2009, Reverdy2017} (to be discussed later).
In previous studies~\cite{Takahashi2016, Oyo2017}, one of the authors proposed a simple satisficing value function called risk-sensitive satisficing ($RS$) 
and empirically validated its effectiveness through numerical simulations of reinforcement learning tasks. 

In this paper, we apply $RS$ to the $K$-armed bandit problems, which constitute the most basic class of reinforcement learning tasks, and prove two propositions.
First, we prove that $RS$ is guaranteed to find a satisfactory action: 
if the $RS$ agent chooses an action in each trial and the number of trials is sufficient, the agent can stably choose an action whose value is above the aspiration level.
Second, we prove the finiteness of the regret of $RS$. 
In general, the performance of algorithms in the $K$-armed bandit problems is measured by how small their regret (expected loss) is.
It is known that the regret increases at least in the logarithmic order with the number of trials~\cite{Lai1985}.
Therefore, the regret increases infinitely as the trials are repeated.
However, we prove that if a small amount of information on the reward distributions is available so that the aspiration level is set to an ``optimal level'' (hence, satisficing entails optimizing), then the regret of $RS$ is upper bounded by a finite value.
We confirm these results by numerical simulations and compare the performance of $RS$ with that of other representative algorithms for the $K$-armed bandit problems. Finally, we conclude the paper with a discussion on the possible applications of $RS$ and the theoretical significance of this work. 

\section*{Methods}

\subsection*{$K$-armed Bandit Problems}

The $K$-armed bandit problems that we deal with in this paper are as follows.
Let there be $K$ actions $\{a_1, a_2, \dotsc, a_K \}$ that lead to a reward of 1 or 0 according to the reward probabilities $\{p_1, p_2, \dotsc, p_K \}$, which are unknown to the agent.
If the agent chooses action $a_i$, 
it acquires a reward of 1 with probability $p_i$ or a reward of 0 with probability $1-p_i$. 
The goal of the repetition of choice is maximization of the expected accumulated rewards, which is measured by minimization of \textit{regret} (the expected cumulative loss).
$a_i^{\ast}$ denotes the action with the maximal reward probability (i.e., $p_{i^{\ast}}  =  \max_ i p_i$).
The regret when the $n$-th step (one step means one trial) ends is defined as follows.
\begin{equation}
\mathrm{regret}(n) = \sum_{i = 1}^{K}(p_{i^{\ast}} - p_i) E[n_i (n)],
\end{equation}
where $n_i(n)$ is the number of times action $a_i$ is chosen from the first to the $n$-th step 
(simply written as $n_i$ when the number of steps is not explicitly indicated)
and $E[\,\cdot \,]$ is the expectation.
Regret represents the expected loss, i.e., ``how inferior the cumulative expected reward from the actual chosen actions is to the cumulative expected reward when the optimal action continues to be chosen from the first step?''
The smaller the regret, the better is the performance of the algorithms.
The minimum value of the regret is zero when the optimal action has been chosen in all the steps. 
It has been proven that the regret increases at least in $\mathcal{O}(\log n)$ with the number of steps $n$~\cite{Lai1985}.

As for action selection by the agent, the basic policy is to take the action with the highest value (the greedy method).  
The basic valuation of action $a_i$ is based on its mean reward:
\begin{equation}
E_i = n_i^1 / (n_i^1 + n_i^0),
\end{equation}
where $n_i^r$ is the number of times $a_i$ is chosen and the reward $r$ is acquired.
$n_i$, i.e., the number of times the action $a_i$ is chosen, satisfies $n_i = n_i^1 + n_i^0$ and $n = \sum_{i = 1}^{K} n_i$.
Under the greedy method with the mean reward valuation, if there is a non-optimal action $a_i \,\,\, (i \neq i^{\ast})$ that has a high value in early trials, there is a risk of $a_i$ being chosen all along.
Each of the other actions must be tried for an appropriate number of times so that the optimal action is found in a timely manner. 
Merely choosing the action with the highest value based on the accumulated knowledge (\textit{exploitation}) does not suffice, and various actions must be tried (\textit{exploration}). 
Various algorithms have been proposed to balance exploitation and exploration. 

\subsection*{Models of Satisficing}
\label{sec:satisficing}
We introduce two models of satisficing at the levels of policy and value function. 
The policy model follows the standard description of satisficing. The second model is the risk-sensitive value function that we analyze and test in this paper. The former is tested through simulations for comparison with the latter. 

\subsubsection*{Policy Satisficing ($PS$) Model}

A standard definition of satisficing is to keep exploring until an action whose value is above the aspiration level $R$ is found and to then stop searching and keep choosing the action (exploit).
Satisficing, unlike optimization, can reduce the search cost because it does not involve searching for all actions and deciding on the optimal action.
This is formulated as a policy (of reinforcement learning) as follows.
If there exists at least one action whose mean reward is above the aspiration level $R$, exploitation (following the greedy method) is executed. Otherwise, when the mean reward of all the actions is below the aspiration level $R$, an action is randomly chosen. 
We refer to this algorithm as \textit{policy satisficing} ($PS$).

\subsubsection*{Risk-sensitive Satisficing ($RS$) Value Function}

One of the authors has proposed a value function called \textit{risk-sensitive satisficing} ($RS$) that realizes satisficing action selection  behavior when operated under the greedy policy~\cite{Takahashi2016, Oyo2017} (see Supplementary Information for its relationship with other models). 
Before introducing the model, we first define the difference between the mean reward $E_i$ of action $a_i$ and the aspiration level $R$: 
\begin{equation}
\delta_i = E_i - R. \label{def_delta}
\end{equation}
If there exists a positive $\delta_i$, then the agent will choose such $a_i$ and be satisfied; otherwise, it will be unsatisfied.
$RS$ is defined as follows~\cite{Takahashi2016}:
\begin{equation}
RS_i = n_i \delta_i = n_i (E_i - R). \label{def_RS}
\end{equation}
This value is used under the greedy policy: the agent chooses the action $a_i$ with the maximal $RS_i$ value.

$RS$ integrates two risk-sensitive satisficing behaviors.
When unsatisfied, $RS$ is risk-seeking, 
leading to \textit{optimistic exploration}. 
If $\delta_i< 0$ for all $i$, then actions with smaller $n_i$ are prioritized.
Let $R=0.7$ and let there be two unsatisfactory actions $a_1$ and $a_2$ with $E_1 = 0.4 < E_2 = 0.6$ and
$n_1 = 7, n_2 = 2$. 
Then, $RS_1 = -2.1 < RS_2 = -0.2$; hence, $a_2$ is chosen. 
This preference of a less tried action can be interpreted as the optimistic expectation of the action's actual reward probability $p_i$ being set above $R$. 
There might be some $p_i>R$; however, thus far, 
$E_i<R$ for all the actions. 
In terms of looking for a satisfactory action, 
it is rational to try actions with smaller $n_i$.
This accords with the motto ``optimism in the face of uncertainty,'' which is considered a general and rational exploration strategy in reinforcement learning~\cite{Bubeck2012}. 
The UCB model described later implements this idea~\cite{Auer2002}. 

When satisfied, $RS$ is risk-averse, performing \textit{pessimistic exploitation}. 
If there is only one $a_i$ for which $\delta_i$ is positive, the agent will keep choosing it. 
If there are multiple actions with positive $\delta_i$, then the actions with larger $n_i$ are prioritized.
Let $R=0.3$, and let there be two satisfactory actions $a_1$ and $a_2$ with $E_1 = 0.4 < E_2 = 0.6$ and $n_1 = 7, n_2 = 2$ that are equivalent to the example above. 
Then, $RS_1 = 0.7 > RS_2 = 0.6$; hence, $a_1$ is chosen. 
In this case, a more tried action is preferred. 
This can be interpreted as the pessimistic expectation of the action's actual reward probability $p_i$ being set below $R$. 
It is possible that $a_i$ is a spuriously satisfactory action with $E_i > R$; however, $p_i < R$. 
In terms of looking for a truly satisfactory action 
and avoiding spuriously satisfactory ones, 
it is rational to try actions with $E_i > R$ for a larger $n_i$.

\subsubsection*{Setting of the Aspiration Level}

The aspiration level $R$ defines the boundary between satisfactory and unsatisfactory, analogous to the break-even point between gain and loss or the neutral reference outcome in prospect theory~\cite{Tversky1981}. 
It can be set according to the internal need for it or its knowledge of the environment.
As an ecological example, let the agent be an animal, and let the rewards 1 and 0 represent the presence and absence of food.
If the action is to look for food at a feeding ground from among multiple grounds and the agent has to obtain food around once every two days for survival, then $R$ would be $0.5$ or higher.

Optimization can be viewed as a special case of satisficing. 
If $R$ lies between the two reward probabilities of the optimal and second-optimal actions, then satisficing above $R$ means optimizing. 
Let us call such $R$ ``an optimal aspiration level''.
Let the highest reward probability be $p_{\mathrm{1st}}$ and the second-highest one be $p_{\mathrm{2nd}}$. 
$R$ can be set optimally as follows:
\begin{equation}\label{optR}
R = (p_{\mathrm{1st}} + p_{\mathrm{2nd}}) / 2.
\end{equation}
It is known that the regret increases at least in $\mathcal{O} (\log n)$ with the number of steps $n$~\cite{Lai1985}.
This is the result of assuming no knowledge of the agent on the reward distribution.
By relaxing this assumption and allowing $R$ to be set as in Eq. \ref{optR}, it will be shown that the regret is upper bounded by a finite value as in Proposition \ref{prop:regret} described later. 

Note that having an optimal aspiration level does not make a $K$-armed bandit problem trivial. 
Even if we know a point between the optimal and second-optimal actions, we do not know exactly which action is optimal. Efficient identification of such an action is not trivial. 
In the next section, $RS$ will be compared in terms of its performance with other algorithms, one of which needs some similar information on the reward distribution to be optimal. 

% \subsection*{Data availability}
% All data are generated by numerical simulations and they have all been reported in the paper.

\section*{Results}

\subsection*{Analysis}

We perform theoretical analysis of the basic satisficing and optimizing properties of $RS$.
First, in Proposition \ref{prop:guarantee}, we prove that $RS$ can stably choose actions above the aspiration level after a sufficient number of steps.
Second, in Proposition \ref{prop:regret}, we prove that the regret of $RS$ is upper bounded when an optimal aspiration level is given and satisficing becomes optimizing.

\subsubsection*{Guarantee of Satisficing}

In the proof of Proposition \ref{prop:guarantee}, 
we adopt symbols clearly indicating the step number ($s$) and the chosen action ($a_i$) as follows.
Both of the following represent values after $s$ steps: 
the mean reward 
\begin{equation}
E(a_i,s) = \frac{n_i^1(s)}{n_i(s)} \\
\end{equation}
and the $RS$ value
\begin{equation}
RS(a_i,s) = n_i(s) \cdot \bigl(E(a_i, s) - R \bigr).
\end{equation}

\begin{prop}[Theoretical Guarantee of Satisficing]
\label{prop:guarantee}
Let $p_i$ be the reward probability of action $a_i$ $(i=1, 2, \dotsc, K)$.
Let $A_U$ be the set of actions whose reward probability is greater than the aspiration level $R$, and let $A_L$ be the set of actions whose reward probability is smaller than $R$. 
Let $I_U = \{ i \mid p_i > R \}$, $I_L = \{ i \mid p_i < R \}$ and $A_U = \{ a_i \mid i \in I_U \}$, $A_L = \{ a_i \mid i \in I_L \}$, where $A_U$ is supposed to be a non-empty set.
Then, the following holds for $RS$.

\begin{center}
After a sufficient number of steps, a satisfactory action $a_i$ with $p_i > R$ will be always chosen, 
and this state is stable.
\end{center}

In other words, by letting 
$P(A)$ be the probability that event $A$ will occur, 

\begin{equation}
\label{eq:prob_guarantee}
P \Bigl(\argmax _{a_i} RS(a_i, s) \in A_U \Bigr) = 1 \,\,\, (s \rightarrow \infty).
\end{equation}
\end{prop}

Subsequently, by $N_j = \Bigl \{ s\Bigm|\argmax \limits _{a} RS(a, s) = a_j \Bigl \} $, we denote the set of steps in which action $a_j$ is chosen.
Let $\# N$ be the number of elements in set $N$.
First, we prove two claims. 

\begin{claim}
\label{claim:lower}
\begin{equation}
\forall \,\,\, i \in I_L, \,\,\,P \bigl( \# N_i = \infty \Leftrightarrow RS(a_i, s) \rightarrow -\infty \,\,\, (s \rightarrow \infty)\bigr) = 1.
\end{equation}
\end{claim}

\begin{proof}{(Claim \ref{claim:lower})}
($\Leftarrow$) 
Suppose that $i \in I_L$ \text{and} $RS(a_i, s) \rightarrow -\infty \,\,\, (s \rightarrow \infty)$.
If $\# N_i < \infty$, $RS(a_i, s)$ is constant for $s$ greater than or equal to some number. 
This is a contradiction; hence, we have $\# N_i = \infty$.
($\Rightarrow$) 
Suppose that $i \in I_L$ \text{and} $\# N_i = \infty$.
By the law of large numbers, for any positive number $\epsilon$, there exists some $S$ such that we have $P \bigl( |E(a_i, s) - p_i| <  (R - p_i) / 2 \bigr) > 1 - \epsilon$  for any integer $s \in N_i$ greater than $S$.
Now, if $|E(a_i, s) - p_i| < (R - p_i) / 2$, we have
\begin{align}
RS(a_i,s) &= n_i(s) \cdot \bigl(E(a_i, s) - R \bigr) \notag\\
                 &< n_i(s) \cdot \bigl( p_i + \frac{R - p_i}{2} - R \bigr) \notag \\
                 &= n_i(s) \cdot \frac{p_i - R}{2} < 0. \label{eq:claim1}
\end{align}
As $s \rightarrow\infty$, we have $n_i(s)(p_i - R)/2 \rightarrow - \infty$; hence, $RS(a_i, s) \rightarrow - \infty$.
Therefore, $P \bigl( RS(a_i, s) \rightarrow - \infty \bigm| \#N_i = \infty\bigr) > 1- \epsilon$.
Since $\epsilon$ is arbitrary, we obtain $P \bigl( RS(a_i, s) \rightarrow - \infty \bigm| \#N_i = \infty\bigr) = 1$.
\end{proof}

\begin{claim}
\label{claim:upper}
\begin{equation}
\exists i \in I_U, \,\,\, P(\#N_i = \infty) = 1.
\end{equation}
\end{claim}

\begin{proof}{(Claim \ref{claim:upper})}
We assume that for any $i \in I_U$, $\# N_i < \infty$.
Then, for any $i \in I_U$, $RS(a_i, s)$ is constant for any $s$ greater than or equal to some number.
Furthermore, for some $j \in I_L$, we have $\# N_j = \infty $.
Hence, by Claim \ref{claim:lower}, we have
\begin{align}
P \bigl( \exists j \in I_L, \,\,\, RS(a_j, s) \rightarrow - \infty \bigm| \forall i \in I_U, \,\,\, \# N_i < \infty \bigr) = 1.
\end{align}
However, the following statements contradict each other: (i) $RS(a_j, s) \rightarrow - \infty$, (ii) $\forall i \in I_U$, $RS(a_i, s) = \mathrm{const.}$ for any $s$ greater than or equal to some number.
Hence, we obtain
\begin{align}
P \bigl( \exists j \in I_L, \,\,\, RS(a_j, s) \rightarrow - \infty\,\,\,\,\text{and}\,\,\,\, \forall i \in I_U, \,\,\, \# N_i < \infty \bigr) = 0.
\end{align}
Now, the following formula holds.
\begin{align}
   &P \bigl( \exists j \in I_L, \,\,\, RS(a_j, s) \rightarrow - \infty\,\,\,\,\text{and}\,\,\,\, \forall i \in I_U, \,\,\, \# N_i < \infty \bigr) \notag \\
   &= P \bigl( \exists j \in I_L, \,\,\, RS(a_j, s) \rightarrow - \infty \bigm| \forall i\in I_U, \,\,\, \# N_i < \infty \bigr) P(\forall i \in I_U , \,\,\, \# N_i < \infty).
\end{align}
Therefore, we must have $P(\forall i \in I_U, \,\,\, \# N_i < \infty) = 0$.
\end{proof}

\setcounter{prop}{0}
\begin{prop}[again]
\label{prop:guarantee_proof}
\begin{equation}
P \Bigl(\argmax _{a_i} RS(a_i, s) \in A_U \Bigr) = 1  \,\,\, (s \rightarrow \infty). \tag{\ref{eq:prob_guarantee}}
\end{equation}
\end{prop}

\begin{proof}{(Proposition\ref{prop:guarantee})}
By Claim \ref{claim:upper}, we have $\exists k \in I_U$, $\# N_k = \infty$.
By the law of large numbers, for any positive number $\epsilon$, there exists some $S$ such that we have $P \bigl( |E(a_k, s) - p_k| <  (p_k  - R) / 2 \bigr) > 1 - \epsilon$ for any integer $s \in N_k$ greater than $S$. 
Now, if $|E(a_k, s) - p_k| <  (p_k  - R) / 2$, we have
\begin{align}
RS(a_k,s) &= n_k(s) \cdot \bigl(E(a_k, s) - R \bigr) \notag\\
                 &> n_k(s) \cdot \bigl( p_k + \frac{R - p_k}{2} - R \bigr) \notag \\
                 &= n_k(s) \cdot \frac{p_k - R}{2} > 0.
\end{align}
Hence, we have $P \bigl(\text{for sufficiently large }s, \,\,\, RS(a_k, s) > 0 \bigr) > 1-\epsilon$.
Since $\epsilon$ is arbitrary, we obtain $P \bigl(\text{for sufficiently large }s, \,\,\,  \allowbreak RS(a_k, s) > 0 \bigr) = 1$.

Here, we assume that there exists $i \in I_L$ such that $\# N_i = \infty$.
Then, we may have $RS(a_i, s) \rightarrow - \infty$ by Claim \ref{claim:lower}.
On the other hand, $\# N_i < \infty$ follows from $RS(a_i, s) \rightarrow - \infty$  because $RS(a_k, s) > 0$ for any sufficiently large $s$.
However, $\# N_i = \infty$ and $\# N_i < \infty$ contradict each other,
which means that the initial assumption must be false.
Hence, for any $i \in I_L$, $P( \# N_i < \infty) = 1$ holds. Therefore, the results obtained are summarized as $\exists k \in I_U$, $P( \# N_k = \infty) =1$ and $\forall i \in I_L$, $P( \# N_i < \infty) = 1$. From these results, the following follows immediately.
$\displaystyle P \Bigl(\argmax _{a_i} RS(a_i, s) \in A_U \Bigr) = 1  \,\,\, (s \rightarrow \infty)$.
\end{proof}

\subsubsection*{Theoretical Analysis of Regret}

We prove that $RS$ is upper bounded by a finite value when the level $R$ is set to the optimal aspiration level.

\begin{prop}[Finiteness of Regret of $RS$]
\label{prop:regret}

Let the highest reward probability of all the actions be $p_1$ and the second-highest reward probability be $p_2$.
Further, we set $R$ as $R = (p_1 + p_2) / 2$ (an optimal aspiration level). Then, the following holds for $RS$: 
\begin{center}
``There exists a monotonically increasing function $f(s)$ for step number $s$ such that $regret(s) < f(s)$. Then, $f(s) \rightarrow M \,\,\, (s \rightarrow \infty)$, where $M$ is constant.
Thus, $\mathrm{regret}(s) < M$''.
\end{center}
\end{prop}

We conceived the following proof by referring the papers\cite{Kim2015, Kim2016, Kim2018} on TOW (tug-of-war) dynamics model (hereinafter simply referred to as TOW).
TOW is similar to $RS$ (See Supplementary Information for the similarities and differences between $RS$ and TOW).
However, in their paper, the analysis of the finiteness of the regret by TOW was strictly limited to cases in which there are only two actions and the variances of the reward probabilities are equal.
In the case of the bandit problems with the reward following the Bernoulli distributions, equal variance implies $p_1 = p_2$ or $p_2 = 1 - p_1$. (Let $V_i$ be the variance of action $a_i$. $V_1 = V_2 \Leftrightarrow p_1(1 - p_1) = p_2(1 - p_2) \Leftrightarrow (p_1 - p_2 ) \{1 - (p_1 + p_2) \} = 0  \Leftrightarrow p_1 = p_2$ or $p_2 = 1 - p_1$.) Thus, the equal variance is a strong assumption.
Here, we generalize the proof to prove finite regret with $K$ arms ($K \geq2$) and without assuming equal variance.

\begin{proof}{(Proposition\ref{prop:regret})}
Suppose that $p_1 > p_2 > p_i \,\,\, (i \neq 1, 2)$.
Let $RS(a_i, s) = n_i(s) \cdot \bigl(E(a_i, s) - R \bigr) \,\,\, (i = 1, 2, \dotsc, K)$.
The expectation $E$ and the variance $V$ of $RS(a_i, s)$ are $E[RS(a_i, s)] = n_i(s) ~(p_i - R)$ 
and $V[RS(a_i, s)] = n_i(s) \sigma_i^2$, respectively, where $\sigma_i^2 = p_i(1 - p_i)$. 

\noindent Note that
\begin{align}
RS(a_i, s) &= n_i(s) \cdot \bigl (E(a_i, s) - R \bigr) \notag \\
&= n_i^1(s) - n_i(s) R \notag \\
&= (X_{i,1} - R) + (X_{i,2} - R) + \dotsb + (X_{i,n_i(s)} - R) \label{eq:x_sum}
\end{align}
holds, where $X_{i,j} = 1 \text{\,or \,} 0$, indicating the reward when action $a_i$ was chosen in the $j$-th time.
Let $\Delta RS_i(s) = RS(a_1, s) - RS(a_i, s) \,\,\,(i \neq 1)$. Then,
\begin{align}
E[\Delta RS_i(s)] &= n_1(s) (p_1 - R) - n_i(s) (p_i - R) \notag \\
                                &= \{ (p_1 - p_i) / 2 \} (n_1(s) + n_i(s)) \notag\\
                                & \quad + \{ (p_1 + p_i) / 2 - R \} (n_1(s) - n_i(s)). \\
V[\Delta RS_i(s)]    &= n_1(s) \sigma_1^2 + n_i(s) \sigma_i^2.
\end{align}

\noindent Since $(p_1 + p_2) / 2 = R$,
\begin{align}
E[\Delta RS_i(s)] &= \{ (p_1 - p_i) / 2 \} (n_1(s) + n_i(s)) \notag \\
                                 & \quad + \{ (p_i - p_2) / 2 \}(n_1(s) - n_i(s)).
\end{align}
By Proposition \ref{prop:guarantee}, if the step number $s$ is sufficiently large, then $n_1(s) \rightarrow s$ with probability 1\footnote{This is an approximation. 
Also, it is not mathematically strict to fix $n_i(s)$ when calculating the expected value and the variance of $RS(a_i, s)$, and to assume that the trials are independent, when applying the central limit theorem. 
It is possible that the calculated upper bound of the regret is not accurate due to the errors resulting from the approximation and/or the above-mentioned assumption. However, the validity of the upper bound is empirically confirmed as shown in Fig.~\ref{fig1} and~\ref{fig2}.}.

% \footnote{This is an approximation. It is possible that the calculated upper bound of the regret is not accurate due to the approximation error. However, the validity of the upper bound is empirically confirmed as in Fig.~\ref{fig1} and~\ref{fig2}.}. 
Hence,
\begin{align}
E[\Delta RS_i(s)] &= \{ (p_1 - p_i) / 2 \} s+ \{ (p_i - p_2) / 2 \} s \notag \\ 
                                 &= \{ (p_1 - p_2) / 2 \} s. \\
V[\Delta RS_i(s)] &\leq (n_1(s) + n_i(s)) \sigma_{1,i}^2 \leq s \sigma_{1,i}^2,  \notag \\
                                 &\text{where} \,\,\, \sigma_{1,i} = \max (\sigma_1, \sigma_i).
\end{align}

By Eq. \eqref{eq:x_sum} and the central limit theorem, $\Delta RS_i(s)$ follows the normal distribution with expectation $E[\Delta RS_i(s)]$ and variance $V[\Delta RS_i(s)]$.
The probability that $\Delta RS_i(s) < 0$ is $Q(E[\Delta RS_i(s)] / \sqrt{V[\Delta RS_i(s)]})$. 
Here, $Q(x)$ is the $Q$-function, which represents the tail distribution function of the standard normal distribution.
Thus, $Q(x) = (1 / \sqrt{2 \pi}) \cdot \int_x^{\infty} \exp(- t^2 / 2) \,dt$.
Let  $P[s = n + 1, I = i]$ be the probability that action $a_i$ is chosen in the $(n+1)$-th step.

Then, $P[s = n+1, I=i]$ is given by
\begin{align}
P[s = n + 1, I = i] &\leq P[RS(a_j, n) \leq RS(a_i, n) \,\,\, (\forall j \neq i)] \notag \\
                               &\leq P[\Delta RS_i(n) \leq 0] \label{eq:prob_negative} \\
                               &= Q\biggl( \frac{ E[\Delta RS_i(n)] }{ \sqrt{V[\Delta RS_i(n)]} } \biggr) \notag \\
                               &\leq Q \biggl( \frac{ (p_1 - p_2) \sqrt{n} }{ 2 \sigma_{1, i} } \biggr)  \notag\\
                               &= Q( \phi_i \sqrt{n}), 
\end{align}
where we set $\phi_i = (p_1 - p_2) / (2 \sigma_{1,i})$.

By using the Chernoff bound $Q(x) \leq (1 / 2) \exp(- x^2 / 2)$, we evaluate the upper bound of the regret.
\begin{align}
E[n_i(n)] &= \sum_{t = 0}^{n - 1} P[s = t + 1, I = i]  \notag \\
                 &\leq \sum_{t = 0}^{n - 1} Q(\phi_i \sqrt{t}) \notag \\
                 &\leq \frac{1}{2} + \sum_{t = 1}^{n - 1} \frac{1}{2} \exp \Bigl( - \frac{ \phi_i^2}{2} t \Bigr) \notag\\
                 &\leq \frac{1}{2} + \int_0^{n - 1} \frac{1}{2} \exp \Bigl( - \frac{ \phi_i^2}{2} t \Bigr) \, dt \notag\\
                 &= \frac{1}{2} - \frac{1}{\phi_i^2} \biggl(\exp \Bigl( - \frac{ \phi_i^2 }{2} (n - 1) \Bigr)  - 1 \biggr) \\
                 &\rightarrow \frac{1}{2}  + \frac{1}{\phi_i^2} \,\,\, (n \rightarrow \infty).
\end{align}
Therefore, 
\begin{align}
\mathrm{regret}(n) &= \sum_{i = 1}^{K} (p_1 - p_i) E[n_i(n)] \notag \\ 
                                     &\leq \sum_{i = 1}^{K} (p_1 - p_i) \biggl \{ \frac{1}{2} - \frac{1}{\phi_i^2} \biggl(\exp \Bigl( - \frac{\phi_i^2}{2} (n - 1 ) \Bigr)  -1 \biggr) \biggr \} \\ 
                                     &\rightarrow \sum_{i = 1}^{K} (p_1 - p_i) \Bigl( \frac{1}{2} + \frac{1}{\phi_i^2}\Bigr) \,\,\, (n \rightarrow \infty) \label{eq:limit_regret}.
\end{align}
This concludes the proof.

\end{proof}

\subsubsection*{When the Aspiration Level $R$ is Variable}
Both of Propositions \ref{prop:guarantee} and \ref{prop:regret} assumed that the aspiration level $R$ is constant.
When $R$ is variable or stochastic, similar propositions can be established just by slightly modifying the previous proofs assuming that $R$ is within a certain range. See Supplementary Information C for the modifications.
The generalization assures that the upper bound of regret stays finite even when $R$ is not initially set $p_2 < R < p_1$ but converges within $p_2 < R < p_1$ after some finite time step. 

\subsubsection*{Empirical Verification}
\label{sec:simulation}
We verify the proven properties through simulations.
As in Proposition \ref{prop:regret}, $R = (p_1 + p_2) / 2$, where $p_1 > p_2 > p_i \, \,\, (i \neq 1, 2)$.
All the results below are the averaged results of 1,000 simulations.
As an additional performance index, we consider \textit{accuracy}, which is the proportion of the simulations in which the algorithm chose the optimal action in each step.
Thus, the accuracy in the $t$-th step is as follows.

accuracy = 
(Number of times action $a_1$ with the highest reward probability $p_1$ is chosen in the $t$-th step)
/ 
(Total number of simulations).

First, we test whether the difference in reward probabilities can be detected, even if the difference is small, when the optimal aspiration level is set for $RS$. 
We test it with $K = 2$ where $(p_1, p_2) = (0.51, 0.49), \, \,\,(0.501, 0.499)$.
The result is shown in Fig.~\ref{fig1}.
The dotted line at the top in Fig.~\ref{fig1} (b) represents the upper bound of the regret shown by Proposition \ref{prop:regret}.
We see that the accuracy nearly reaches 1 after $10^6$ steps, even if the difference is only 0.002 as in $(0.501, 0.499)$.
Moreover, we see that the regret does not exceed the upper bound (Eq. \eqref{eq:limit_regret}) calculated by Proposition \ref{prop:regret}.

\begin{figure}[t]
	\centering
	\includegraphics[width=0.80\linewidth]{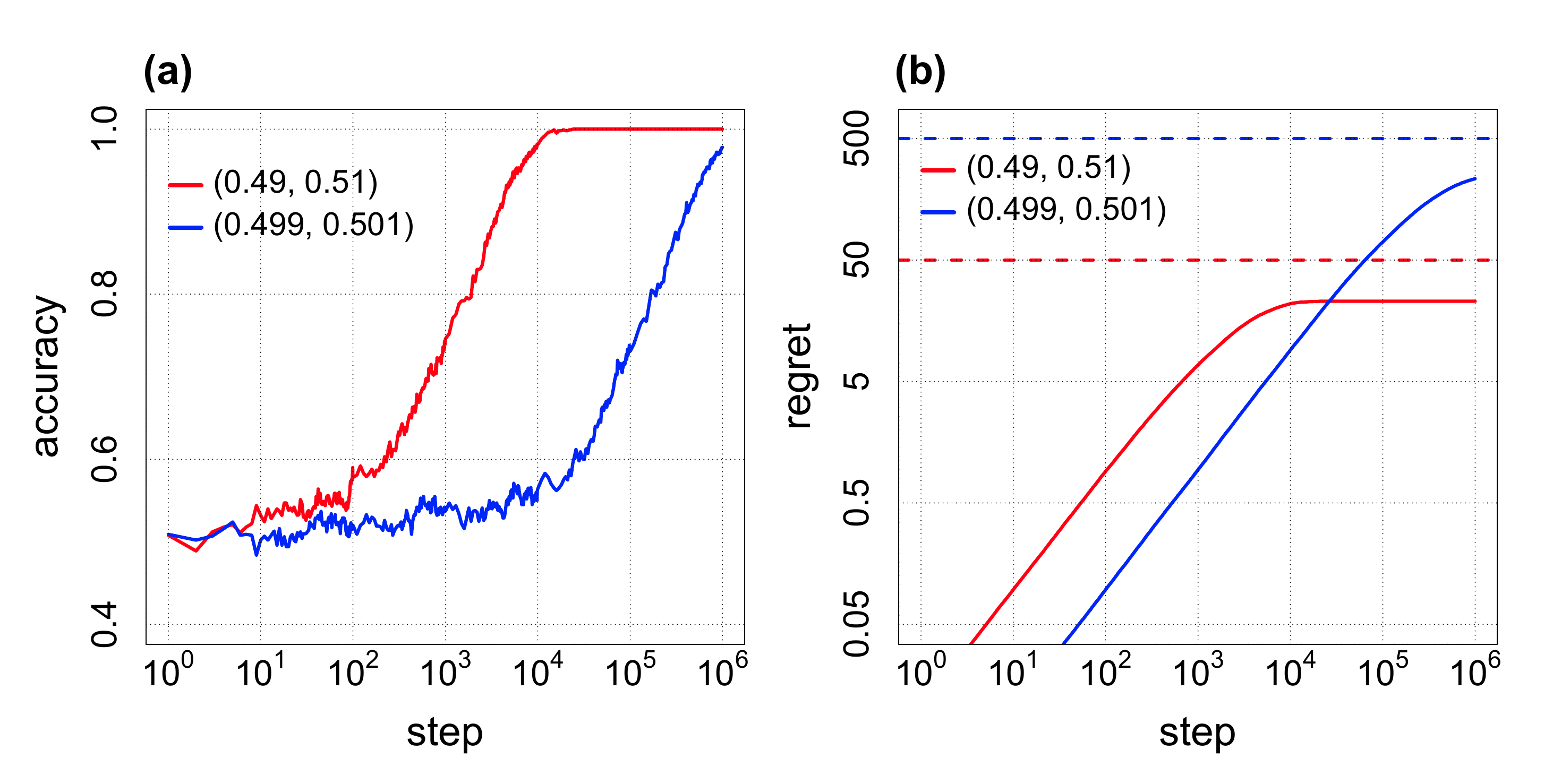}
	 \caption{Simulations of $RS$ with $K = 2$, where the reward probabilities are $(0.51, 0.49)$ or $(0.501, 0.499)$. 
    \textbf{(a)} Plot of accuracy and \textbf{(b)} plot of regret. The dotted line at the top represents the upper bound of the regret calculated by Proposition \ref{prop:regret}.}
    \label{fig1}
\end{figure}

Next, we conduct simulations to confirm the propositions with $K = 10$. 
The reward probability of each action is generated uniformly randomly from $[0, 1]$.
The result is shown in Fig.~\ref{fig2}.
We can see that 
the accuracy converges to 1 and the regret does not exceed the upper bound (Eq. \eqref{eq:limit_regret}) calculated by Proposition \ref{prop:regret}.
Here, the calculated upper bound of the regret for $K = 10$ is considerably higher than the actual regret compared with the case of $K = 2$.
As we evaluate the probability of choosing action $a_i$ only by comparing $a_i$ with action $a_1$ having the highest reward probability as shown in Eq. \eqref{eq:prob_negative} in the proof of Proposition \ref{prop:regret}, the probability of choosing $a_i$ is increasingly overestimated as the number of actions increases. 

\begin{figure}[t]
	\centering
	\includegraphics[width=0.80\linewidth]{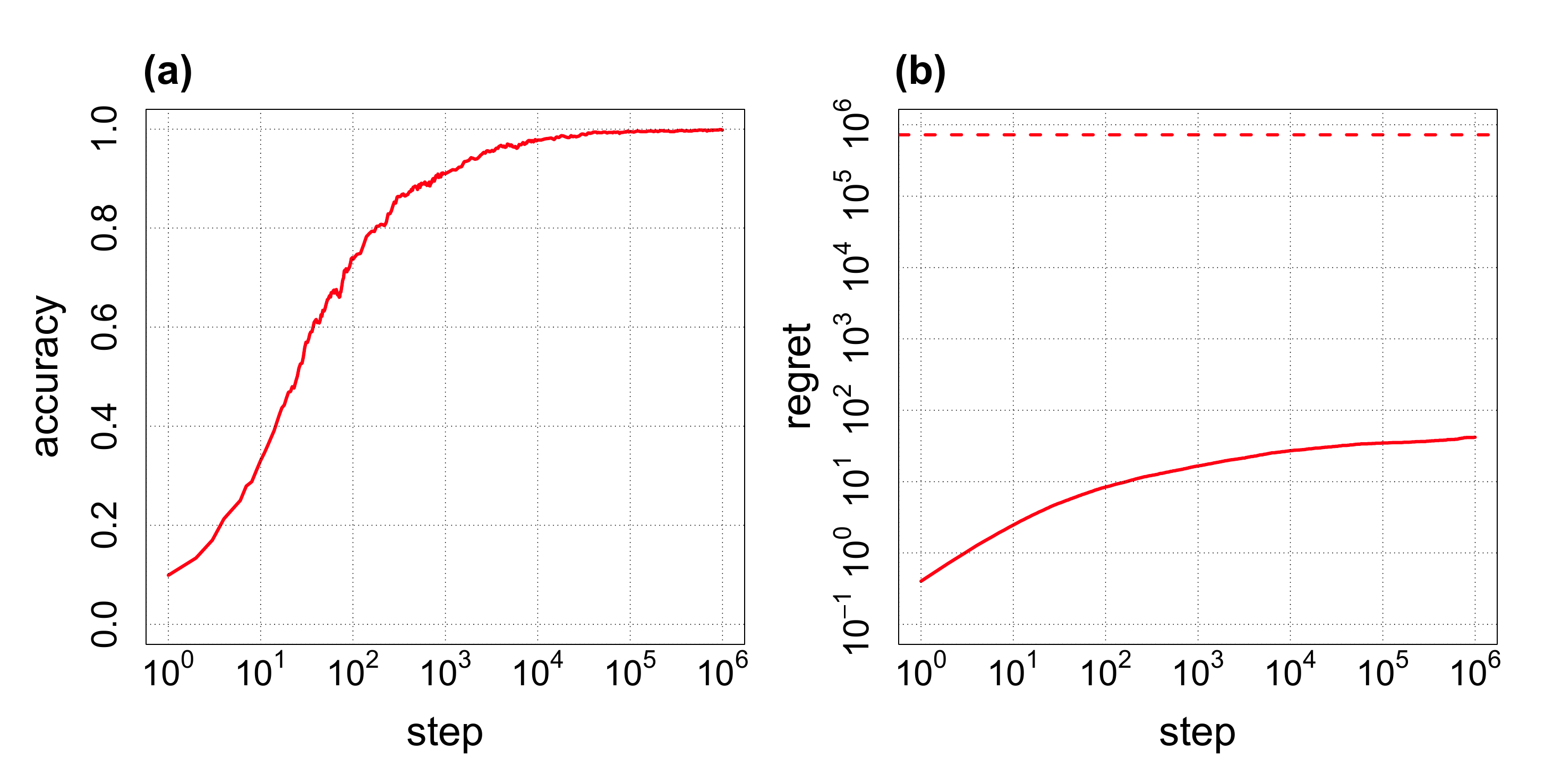}
	 \caption{Simulations of $RS$ with $K = 10$, where the reward probabilities are each generated uniformly randomly from $[0, 1]$ in each simulation. 
     \textbf{(a)} Plot of accuracy and \textbf{(b)} regret. The dotted line at the top shows the upper bound of the regret calculated by Proposition \ref{prop:regret}.}
     \label{fig2}
\end{figure}

\subsection*{Comparison with Other Algorithms}

Here, we clarify the performance and properties of $RS$ by comparing it with some representative algorithms for the $K$-armed bandit problems, namely UCB1-Tuned and $\epsilon_n\text{-greedy}$~\cite{Auer2002}

\subsubsection*{UCB1-Tuned}

Upper confidence bound (UCB) is an algorithm based on the idea that the value of relatively less tried actions (more uncertain) is potentially high, similar to $RS$'s risk-seeking evaluation when unsatisfied~\cite{Auer2002}.
The regret of UCB is guaranteed to increase in the logarithmic order, which is the theoretical limit~\cite{Lai1985}.
We include the result of UCB1-Tuned (hereinafter referred to as UCB1T), which shows better performance compared to UCB1.
\begin{equation}
\mathrm{UCB1T}(a_i) = E_i + \sqrt{\frac{\ln{n}}{n_i} \min\{\frac{1}{4}, V_i (n_i)\}} \, .
\end{equation}
Here, $V_i(n_i) = v_i + \sqrt{2\ln{n} / n_i}$, and $v_i$ is the variance of the reward from choosing action $a_i$.
Further, 1/4 is the upper bound of the variance of the random variable following the binomial distribution.
In the algorithm, the action with the highest UCB1T value is chosen (the greedy method).
The first term $E_i$ of UCB1T, which is the mean reward, represents the already acquired knowledge (and its \textit{exploitation}), whereas the second term, which decreases as action $a_i$ is tried more, expresses the (un-)reliability of $E_i$ (which leads to \textit{exploration}).
When $n_i = 0$, the second term cannot be calculated, but in the first $K$ steps, each action is chosen once so that the value of the second term for all the actions is subsequently finite.

\subsubsection*{$\epsilon_n\text{-greedy}$} 
\label{sec:en_greedy}
To set the level $R$ such that satisficing  implies optimization, it is necessary to have some point in the interval between the highest and second-highest reward probabilities, usually unknown to the agent. 
Thus, having such ``optimal'' $R$ is a type of ``cheating''. 
However, when such information is available, it should be utilized well, and $RS$ does so. 

Furthermore, there is another algorithm, namely $\epsilon_n\text{-greedy}$~\cite{Auer2002}, which requires similar information for optimal performance. 
In this algorithm, the probability of random action selection, $\epsilon_n$, is gradually reduced by annealing so that the regret of $\epsilon_n\text{-greedy}$ is guaranteed to be of the logarithmic order.
It starts with maximal exploration (random action selection) and then gradually shifts to more exploitation as the information of the environment gets accumulated. 
In $\epsilon_n\text{-greedy}$, there are two parameters $c$ and $d$ that are set as $c > 0$ and $0 < d < 1$. 
When there are $K$ arms, the stepwise decreasing sequence $\epsilon_n \in (0, 1], \,\,\, n = 1, 2, \dotsc$ is defined as follows:
\begin{equation}
\epsilon_n = \min \Bigl\{1, \frac{cK}{d^2 n} \Bigr\}.
\end{equation}
The agent chooses action $a_i$ with the highest mean reward with probability $1 - \epsilon_n$, and it chooses a random action with probability $\epsilon_n$ for $n = 1, 2, \dotsc$
Let $p_{\mathrm{1st}}$ be the highest reward probability, and define $\Delta_i = p_{\mathrm{1st}} - p_i$. 
Then, the parameter $d$ needs to satisfy
\begin{equation}
0 < d \leq \min_{i \neq \mathrm{1st}}\Delta_i.
\end{equation}
Further, $\min \Delta_i = p_{\mathrm{1st}} - p_{\mathrm{2nd}}$ needs to be known in advance. Thus, some information about the reward probabilities is required, as in the case of $RS$ with the optimal aspiration level.
In addition, the performance of $\epsilon_n\text{-greedy}$ is sensitive to the value of the parameter $c > 0$, and it is difficult to find the optimal value of $c$~\cite{Auer2002}.

On the other hand, determining the optimal aspiration level $R$ for $RS$ may be easier.  It does not require a parameter like $c$, and $(p_{\mathrm{1st}} + p_{\mathrm{2nd}}) / 2$ is sufficient. 
More generally, it is sufficient to obtain the interval $[p_{\mathrm{2nd}}, p_{\mathrm{1st}}]$  or the value of any point within the interval.

\subsubsection*{Existing Satisficing Models}

Here, we introduce the existing satisficing models and briefly explain the difference between those models and $RS$.
First, the framework that is the closest to ours is that of Bendor et al. on the heuristics of satisficing~\cite{Bendor2009}, which analyzes the two-armed bandit problems when the rewards are Bernoulli distributed.
They mainly analyzed the limiting behavior of the policy model similar to $PS$.
Their model is different from $PS$ in that it gives a probability parameter of switching actions with a certain probability (not always), when unsatisfied.
Therefore, the performance of their model is lower than that of $PS$.

The most recent and comprehensive study was conducted by Reverdy et al.~\cite{Reverdy2017}
They decomposed satisficing into ``satisfy'' and ``suffice'' (from which the word ``satisfice'' is formed) 
and presented general problem settings that include the standard bandit problems and algorithms with optimal order.
As their algorithm is an adaptation of the standard UCB~\cite{Auer2002}, the difference between $RS$ and their algorithm is similar to the difference between $RS$ and UCB as described above.
Furthermore, their analysis is limited to the bandit problems where the reward distributions are Gaussian.
In their study, they extended the concept of regret and developed an algorithm that searches for actions that exceed the aspiration level with probability $(1 - \delta)$.
They proved the finiteness of the regret for their algorithm when $\delta > 0$.

However, it should be noted that in their study, the definition of regret is changed.
Specifically, the regret of their algorithm is calculated according to whether or not the expected reward exceeds the aspiration level with probability $(1 - \delta)$, and the definition that regards the regret occurring with probability $\delta$ as zero is adopted.
If $\delta = 0$, their regret is calculated according to whether the expected reward always exceeds the aspiration level or not; therefore, it becomes the same framework as that of the ordinary bandit problems.
In such cases, the regret of their algorithm increases in the logarithmic order, which is the theoretical limit, and it does not become finite.
On the other hand, $RS$ can achieve the finite regret without changing the definition of regret.
Therefore, the purposes and problem settings are different in our study and their study.

According to the above-mentioned discussion, it is difficult to compare our study with other satisficing algorithms for reinforcement learning proposed in previous studies because the purposes and frameworks are different. It is sufficient to compare our approach with $PS$ and UCB1. Accordingly, the other algorithms will not be handled directly hereafter.

\subsubsection*{Performance Comparison}

We compare the performance of UCB1T, $PS$, $\epsilon_n \text{-greedy}$, and $RS$ with $K = 100$ through numerical simulations.
Furthermore, the reward probabilities are uniformly randomly selected from $[0, 1]$, and the average is over 1,000 simulations.
As mentioned above, it is difficult to determine the parameter $c$ of $\epsilon_n \text{-greedy}$.
In this simulation, the regret of $\epsilon_n \text{-greedy}$ in the 10,000-th step is taken as a reference.
It is empirically found by a long parameter sweep such that the regret of $\epsilon_n \text{-greedy}$ in the 10,000-th step is minimized at around $c = 1 \times 10^{-5}$.
Hence, the results of $c = 1 \times 10^{-6}, 1 \times 10^{-5}, 1 \times 10^{-4} $ are shown as comparison targets.
We set $d$ as $d = p_{\mathrm{1st}} - p_{\mathrm{2nd}}$.
As for $RS$ and $PS$, we set the aspiration level $R$ to an optimal level, 
$R = (p_{\mathrm{1st}} + p_{\mathrm{2nd}}) / 2$, so that we can evaluate the efficiency when satisficing implies optimization.

The results are shown in Fig.~\ref{fig3}.
As for accuracy, $RS$ approaches 1 the fastest among these algorithms.
As for regret, $PS$ increases rapidly because it randomly chooses actions unless an action whose reward is above $R$ is found.
The regret of $RS$ remains small (and bound finitely), whereas UCB1T and $\epsilon_n \text{-greedy}$ diverge at a logarithmic order. 
In summary, we can see that $RS$ with the optimal aspiration level $R$ shows better performance than UCB1T, $PS$, and $\epsilon_n \text{-greedy}$.

\begin{figure}[t]
	\centering
	\includegraphics[width=0.80\linewidth]{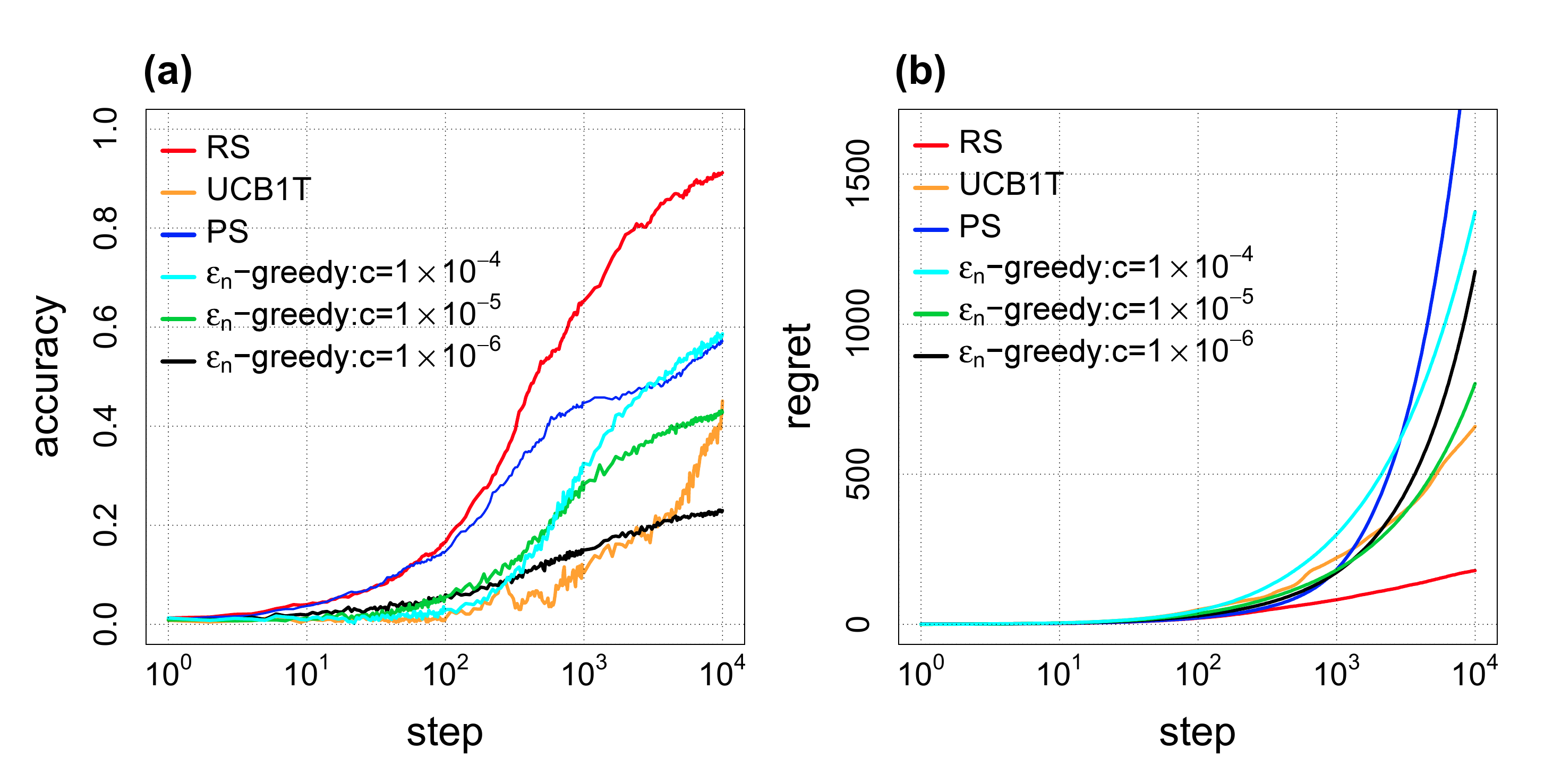}
	 \caption{Simulations with $K = 100$, where the reward probabilities were sampled uniformly. Comparison of $RS$, UCB1T, $PS$, and  $\epsilon_n \text{-greedy}$:  \textbf{(a)} accuracy and \textbf{(b)} regret.}
     \label{fig3}
\end{figure}

\subsubsection*{Analysis of the Expected Change in Value Functions}

Here, we qualitatively consider why $RS$ with the optimal aspiration level $R$ performs better than the other algorithms.
Let us consider how the value of $RS$ in the $n$-th step changes when action $a_i$ is chosen in the $(n+1)$-th step.
In the following $RS$ formula, 
\begin{equation}
RS(a_i, n) = n_i (n) \cdot (E(a_i, n) - R) = n_i^1 (n) - n_i (n)R, 
\end{equation}
$n_i^1 (n)$ is the number of times a reward of 1 is obtained in the choice of action $a_i$ from the first to the $n$-th step.
In the $(n+1)$-th step, the value of $RS$ changes with probability $p_i$ to
\begin{equation}
RS(a_i, n+1)=n_i^1(n) + 1 - (n_i (n) + 1)R = RS(a_i, n) + 1 - R,
\end{equation}
whereas it otherwise changes with probability $(1 - p_i)$ to
\begin{equation}
RS(a_i, n+1) = n_i^1(n) - (n_i(n) + 1)R = RS(a_i, n) - R.
\end{equation}
Let $\Delta RS(a_i, n) = RS(a_i, n+1) - RS(a_i, n)$.
Then, the expected value of the change, $E[\Delta RS(a_i, n)]$, is as follows:
\begin{equation}
E[\Delta RS(a_i, n)] = p_i (1 - R) + (1 - p_i )(- R) = p_i - R.
\end{equation}
Thus, we see that the following relationships hold in any step: 
\begin{align}
\text{If} \,\,\, p_i > R \,\,\, \text{then} \,\,\, &E[\Delta RS(a_i, n)] > 0, \label{deltaRSpos}\\ 
\text{If} \,\,\, p_i < R \,\,\, \text{then} \,\,\, &E[\Delta RS(a_i, n)] < 0. \label{deltaRSneg}
\end{align}
Let $R$ be set to an optimal level. 
Then, relationship \ref{deltaRSpos} means that once the optimal action $a_i$ is chosen, $RS(a_i)$ will keep increasing on average, and it will continue to be chosen. 
On the other hand, relationship \ref{deltaRSneg} means that if a non-optimal action $a_j$ has the highest $RS$ value, and continues to be chosen for a while, then the value keeps decreasing on average. The value for other actions remains invariant. Therefore, at some point, another action than $a_j$ will start to be chosen. 
Further, note that the $RS$ value decreases at an average rate of $p_j - R$. Therefore, on average, the lower the reward probability of an action, the faster the action will stop being chosen, and another action will start being chosen. 

To clarify the idiosyncrasies of $RS$, we carry out similar analyses for other value functions.
First, let us analyze the mean reward.
The value function is $Q(a_i, n) = n_i^1 / n_i (= E_i)$.
When action $a_i$ is chosen, $E[\Delta Q(a_i, n)]$ is given by
\begin{align}
E[\Delta Q(a_i, n)] &= p_i \biggl( \frac{n_i^1 + 1}{n_i + 1} - \frac{n_i^1}{n_i} \biggr) + (1 - p_i) \biggl( \frac{n_i^1}{n_i + 1} - \frac{n_i^1}{n_i} \biggr) \notag \\
&= \frac{p_i n_i - n_i^1}{(n_i + 1)n_i} \notag \\
&= \frac{p_i - E_i}{n_i + 1} \label{eq:delta_V},
\end{align}
whereas the values for other actions do not change.
Further, $E[\Delta Q(a_i, n)]$ is positive if $p_i > E_i$ and negative if $p_i < E_i$, and both cases may occur regardless of the reward probability $p_i$ because $E_i$ is a variable, in contrast to the constant $R$ for $RS$.
If action $a_i$ is chosen for a sufficient number of times, $p_i \approx E_i$ holds.
Then, it leads to $E[\Delta Q(a_i, n)] \approx 0$, and $Q(a_i, n)$ remains nearly unchanged.
This implies that there is a possibility that a non-highest action keeps to be chosen (trapped into a local optimum).
Let us consider the simplest example where there are only two actions (with $p_1 > p_2$), and choosing the optimal action $a_1$ does not give much rewards, leading to $E_1 < p_2$ and $E_1 < E_2$.
As $n_2$ increases, $E_2$ converges to $E_2 \approx p_2$, 
and the relationship of $E_1 < E_2$ becomes fixed because of $E_1 < p_2$.
This leads to $a_2$ being chosen constantly. 
To avoid the local optima, $\epsilon_n\text{-greedy}$ prevents a non-highest action from being continuously chosen by randomly choosing actions with probability $\epsilon_n$.
With the mean reward, unlike $RS$, we cannot  say that the smaller the reward probability of the action chosen once, the faster on average is the switching of the agent to choose another action.

Next, let us analyze UCB1, which is the simplest algorithm in the UCB family. 
\begin{equation}
\mathrm{UCB1}(a_i, n) = E_i + \sqrt{ \frac{2 \ln{n}}{n_i}}. \label{eq:UCB1}
\end{equation}
When action $a_i$ is chosen, the expected change in the UCB1 value is 
\begin{align}
E[\Delta \mathrm{UCB1}(a_i, n)] &= p_i \biggl \{ \frac{n_i^1 + 1}{n_i + 1} + \sqrt{ \frac{2 \ln{(n + 1)}}{n_i + 1}} - \biggl( \frac{n_i^1}{n_i } + \sqrt{ \frac{2 \ln{n}}{n_i}} \, \biggr) \biggr \} \notag \\
& \qquad + (1 - p_i) \biggl \{ \frac{n_i^1}{n_i + 1} + \sqrt{ \frac{2 \ln{(n + 1)}}{n_i + 1}} - \biggl( \frac{n_i^1}{n_i } + \sqrt{ \frac{2 \ln{n}}{n_i}} \, \biggr) \biggr \} \notag \\
&= \frac{p_i n_i - n_i^1}{(n_i + 1)n_i} + \biggl( \sqrt{ \frac{2 \ln{(n + 1)}}{n_i + 1}} - \sqrt{ \frac{2 \ln{n}}{n_i}} \, \biggr) \notag \\
&= \frac{p_i - E_i}{n_i + 1} + \biggl( \sqrt{ \frac{2 \ln{(n + 1)}}{n_i + 1}} - \sqrt{ \frac{2 \ln{n}}{n_i}} \, \biggr), \label{eq:delta_UCB1-i}
\end{align}
whereas the expected change of non-chosen action $a_j$ is as follows: 
\begin{align}
E[\Delta \mathrm{UCB1}(a_j, n)] &= E_j + \sqrt{ \frac{2 \ln{(n + 1)}}{n_j}} - \biggl( E_j + \sqrt{ \frac{2 \ln{n}}{n_j}} \biggr) \notag \\
&= \sqrt{\frac{2}{n_j}}(\sqrt{\ln{(n + 1)}} - \sqrt{\ln{n}}). \label{eq:delta_UCB1-j}
\end{align}
In Eq. \eqref{eq:delta_UCB1-i}, the first term is the same as that in Eq. \eqref{eq:delta_V}.
In Eq. \eqref{eq:delta_UCB1-i}, the second and third terms approach zero if  action $a_i$ continues to be chosen.
Hence, if we consider only Eq. \eqref{eq:delta_UCB1-i}, there is a possibility that the non-highest action continues to be chosen, as with Eq. \eqref{eq:delta_V}.
However, in UCB1, the value function of non-chosen action $a_j$ also changes, as in Eq. \eqref{eq:delta_UCB1-j}.
Moreover, we can see that the value of the non-chosen action increases infinitely because of the second term of Eq. \eqref{eq:UCB1}.
As a result, a non-highest action does not continue to be chosen.

In Eq. \eqref{eq:delta_UCB1-i}, the first term is positive if $p_i > E_i$ and negative if $p_i < E_i$, and both cases may occur regardless of the reward probability $p_i$ because $E_i$ is a variable, as it is for $Q$ above.
On the other hand, the second term between the parentheses is negative if $n \geq 3$, which results from the fact that $f(x) = ( \ln x )/ x$ monotonically decreases with $x > e \,( > 2)$.
As a result, 
$E[\Delta \mathrm{UCB1}(a_i, n)]$ 
may be positive or negative, regardless of the reward probability.
Therefore, UCB1 does not have the property of $RS$ whereby the action with a lower reward probability will be switched from earlier. 

Based on the analyses presented above, let us reconsider the form of 
$RS_i = n_i \delta_i = n_i (E_i - R)$.
Starting from the most basic value function of the mean reward, $E_i$, 
$RS$ is formed through two operations, $(\cdot)-R$ and $n_i(\cdot)$. 
If it is merely $\delta_i$, the value function $\delta_i$ works exactly as the original $E_i$ under the greedy policy.
On the other hand, if only $n_i(\cdot)$ is applied, the value function is $n_i E_i = n_i^1$, and it is a special case of $RS$ with  $R = 0$ where any action is satisfactory.
With $n_i E_i$, the agent will continue to choose the first action that gives a reward of 1.
By applying the two operations, 
we acquire the property of 
$E[\Delta RS(a_i, n)] = p_i - R$, 
the constant change in the $RS$ value, regardless of the step number $n$. 
Therefore, the $RS$ value of an unsatisfactory action (with the reward probability below the aspiration level) constantly decreases on average; as a result, the action will cease to be chosen at some point.
Furthermore, we can say that the smaller the reward probability of the action chosen once, the faster on average is the switching of the $RS$ agent to the choice of other actions.
As shown above, UCB and $\epsilon_n\text{-greedy}$ have no such property.
Therefore, this property is considered to be one of the reasons why the performance of $RS$ using the optimal aspiration level is superior to that of other basic algorithms.

\section*{Discussion}
\label{sec:discussion}
In this paper, we introduced a simple model called $RS$ that implements a satisficing strategy for the $K$-armed bandit problems, which constitute one of the most basic classes of reinforcement learning tasks. 
We proved two propositions. 
One is that $RS$ is guaranteed to find a satisfactory action with the reward probability above the aspiration level. 
The other is that the regret (expected loss) of $RS$ is upper bounded by a finite value when an optimal aspiration level (where satisficing implies optimizing) is given.
Then, we confirmed the results through numerical simulations and compared the performance of $RS$ with that of other representative algorithms for the $K$-armed bandit problems. 
In addition, we analyzed the property of $RS$ relative to other algorithms and validated why $RS$ has its own form. 

Except in Proposition \ref{prop:guarantee}, we assumed that we can set the aspiration level $R$ to an optimal level. 
As the optimal aspiration is not always available to the agent, 
a future research direction would be to develop an algorithm that can learn an optimal aspiration level $R$ online.
As a preliminary result, an algorithm that exploits the properties of $RS$ has shown performance comparable to that of Thompson sampling \cite{Agrawal2012}, although it has not been theoretically guaranteed thus far~\cite{KonoTakahashiJSAI2018}. 

There are many other advantages of $RS$ besides those mentioned in this paper.
For example, the satisficing behavior is scalable in the sense that its performance does not depend on the scale of the problems, such as the number of actions, but rather on the proportion of satisfactory actions, unlike optimization algorithms~\cite{Oyo2017}.
In addition, as $RS$ is a simple value function without assumptions such as the family of reward probability distributions, 
it can be applied to other reinforcement learning tasks through some straightforward generalization.
In fact, it has been shown that the generalized $RS$ can conduct autonomous and efficient searches in a robotic motion learning task in which a robot learns to perform giant swings (acrobot)~\cite{Takahashi2016}.

One of the computational advantages of satisficing, compared to optimization, is that it can convert an optimization problem into a decision problem. 
With $RS$, the guaranteed satisficing algorithm, and $R$ at a certain level, we can efficiently determine whether there is an action whose value is above $R$. 
The decision framework is especially useful when a certain level of reward, rather than the optimal level, is necessary. 
It also facilitates parallelization. 
For example, we can set the aspiration levels $R_1,  R_2, \dotsc, R_N$ to $N$ agents in ascending order, respectively, and make the agents execute a certain task in parallel.
If the task succeeds at the level $R_i$ and fails at the level $R_{i+1}$, we can see that the optimal solution exists somewhere in $[R_i, R_{i+1}]$, and the interval may be incrementally narrowed down.
This is somewhat close to human learning for solving a task. 
When trying to solve a task, we usually do not randomly try and err in a purely bottom-up manner. 
Instead, we tend to adopt a top-down constraint in our trials, such as trying to run one mile in four minutes. 
Guaranteed satisficing may lead to reinforcement learning methods that solve  tasks somewhat similarly to humans.

% \bibliography{MyCollection}
\bibliography{output}

\section*{Acknowledgements}

This work was partially supported by JSPS KAKENHI Grant Number 17H04696. 

\section*{Author contributions statement}

T.T. and A.T. conceived the analyses. A.T. conducted the proofs and experiments. Both the authors wrote and reviewed the manuscript. 

\section*{Additional information}

\textbf{Competing interests:} The authors declare no competing interests.

\section*{Supplementary information}
%In this supplementary material, two distinctive aspects of $RS$ are discussed.
In this supplementary material, two distinctive aspects of $RS$ and a generalization of the two propositions in the main text are discussed.
First, we show that $RS$ can be considered as a generalization of another model, $S0$~\cite{Shinohara2007}.
$S0$ in the bandit setting is based on the premise that high performance can be achieved through competitive evaluation of actions. 
However, our generalization from $S0$ to $RS$ shows that competitive evaluation appears only in the two-armed settings, and in general (in the $K$-armed settings), the fundamental is the risk-sensitive satisficing behavior. 
Second, we compare $RS$ with the Tug-of-war (TOW) dynamics models~\cite{Kim2015, Kim2016, Kim2018}, which was referred to in the proof of Proposition \ref{prop:regret}. 
TOW is based on the notion of conservation of physical quantities, and it leads to competitive evaluation. 
We show that, under certain conditions, $RS$ has the same mathematical form as a part of the recent TOW dynamics models. 
In addition, $S0$ and TOW are both limited in terms of their application to the evaluation of only two actions (or two classes of actions). 
On the other hand, as materialized in $RS$, the notion of risk-sensitive satisficing enables generalization (to an arbitrary number of actions), simplification, conceptual clarity, and high performance in terms of satisficing, as suggested in the main text of the paper.
Third, we slightly generalize Propositions \ref{prop:guarantee} and \ref{prop:regret} and their proofs in the main text assuming that the aspiration level $R$ is within a certain range.

\subsection*{A. $RS$ as a generalization of  $S0$}
\label{sec:S0}

First, we show that $RS$ is a generalization of another value function $S0$, from the number of actions $K = 2$ to arbitrary $K \geq 2$ and from constant aspiration level 0.5 to variable $R \in [0, 1]$.
$RS$ discussed in this paper was formerly called \textit{reference satisficing}~\cite{Takahashi2016, Oyo2017}. 
It was subsequently renamed as \textit{risk-sensitive satisficing} to characterize it more specifically, and abbreviated invariantly as $RS$. 
$RS$ contains $S0$ model in a special form, which was first introduced by Shinohara et al. as a causal reasoning model~\cite{Shinohara2007}. 
The $S0$ model was later termed as the $RS$ (rigidly symmetric) model~\cite{Nakano2008}, and was then used as a value function~\cite{Takahashi2011} in the bandit problems.
Subsequent studies applied $S0$ in the two-armed bandit problems, and the performance of $S0$ was found to be similar to that of $LS$~\cite{Takahashi2011}, which is a more complicated model.
An analysis of these behaviors from a satisficing perspective was first published in 2013~\cite{Oyo2013a, Oyo2015}.
The aspiration level for satisficing was made variable in 2011 \cite{Oyo2011}. Subsequently, in 2012\cite{Kohno2012}, its generalization from two to any arbitrary number of actions of the model was proposed. 
However, $LS$ is much more complicated than $RS$, and the analysis was rather indirect.
Hereafter, we show the equivalence of $RS$ and $S0$ under certain conditions (for two actions with $R=0.5$).

Let $A$ and $B$ be actions in a two-armed bandit problem.
Let $a_X^1$ be the number of times the choice of action $X \in \{A, B \}$ has given reward 1, and let $a_X^0$ be the number of times the choice of action X has given reward 0 (no reward).
Thus, the mean reward is $a_X^1 / (a_X^1  +  a_X^0)$.
Here, $S0$ defines the values of actions $A$ and $B$ as follows:
\begin{align}
S0(A) &= \frac{a_A^1 + a_B^0} {a_A^1 + a_B^0 + a_A^0 + a_B^1 }, \\
S0(B) &= \frac{a_B^1 + a_A^0} {a_B^1 + a_A^0 + a_B^0 + a_A^1 }.
\end{align}
These comparative evaluations identify both the obtaining of reward from action $A$ and not obtaining of reward from action $B$.
Hence, $S0(B) = 1 - S0(A)$ holds.
Because the denominator is common, the comparison of the two values eventually results in the selection of action $A$ if the following inequality holds; if the inequality does not hold, action $B$ is selected:
\begin{equation}
a_A^1 + a_B^0  > a_B^1 + a_A^0. \\  \label{eq:A1B0gB1A0}
\end{equation}
From the above inequality, we can see that transitive law is established when adding action $C$.
That is, let the $S0$ evaluation of $A$ in comparison with $B$ be represented as $S0_{AB}(A)$. If $S0_{AB}(A) < S0_{BA}(B)$ and $S0_{BC}(B) < S0_{CB}(C)$, then $S0_{AC}(A) < S0_{CA}(C)$ .
Thus, we see that the comparable number of actions is not necessarily $K = 2$.
The inequality \eqref{eq:A1B0gB1A0} can be expressed as 
\begin{equation}
a_A^1 - a_A^0 > a_B^1 - a_B^0.
\end{equation}
Using the notations presented in this paper, $a_X^1 = n_X E_X$ and $a_X^0 = n_X (1 - E_X)$ holds.
Then,
\begin{alignat}{2}
              && n_A E_A - n_A (1 - E_A ) &> n_B E_B - n_B (1 - E_B) \\
\Leftrightarrow&&  n_A (2E_A - 1) &> n_B (2E_B - 1) \\
\Leftrightarrow&&  n_A (E_A - 0.5) &> n_B (E_B - 0.5). \label{RS_05_AB}
\end{alignat}   
It can be seen that both sides of inequality \eqref{RS_05_AB} are identical to the form of $RS$ (equation (4) in the main article) with $R = 0.5$.
Because the value of a set of arbitrary actions can be totally ordered thanks to the property of transitivity, it is only necessary to calculate the $RS$ value for each action, independently of all the other actions, and choose the action with the maximum value.

\subsection*{B. Comparison of $RS$ and TOW}
\label{sec:TOWvsRS}

We referred to the TOW dynamics model~\cite{Kim2015, Kim2016, Kim2018} (hereafter simply referred to as TOW) in the proof of Proposition \ref{prop:regret}. 
Here, we compare $RS$ and TOW, and describe the relative advantages of $RS$ over TOW. 
There are many variations of TOW, starting from around 2010~\cite{Kim2010}. 
Here, we focus on recent papers ~\cite{Kim2016, Kim2018} where the proposed model of TOW is the closest to that of $RS$.
Let $X_{k, i}$ be a random variable, representing 
the reward obtained by the $i$-th choice of the action $k$.
%Therefore, $S_k$, which represents something like the value of action $k$ in TOW, can be expressed as
Something like the value of action $k$ in TOW can be expressed as
\begin{align}	
S_k &= X_{k,1} + X_{k,2} + \dotsb + X_{k, n_k} - Kn_k\notag \\
        &= (X_{k,1} - K) + (X_{k,2} - K) + \dotsb + (X_{k, n_k} - K),  \label{eq:TOW_Si}
\end{align}
where $K$ is a parameter.

Let $n_k$ be the number of time action $k$ is chosen, and $E_k$ be the average rewards obtained by choosing action $k$, such that $E_k = \sum_{i = 1}^{n_k} X_{k, i} / n_k$. 
Although in the main text of the paper, the probability distributions of the rewards were assumed to be the Bernoulli distributions, herein, the distribution does not necessarily have to be Bernoulli.
The value function $RS_k$ of the action $k$ of $RS$ is equivalent to the following form, as given in the proof of Proposition \ref{prop:regret}:
\begin{align}
RS_k &= n_k (E_k - R) \notag \\
          &= \sum_{i = 1}^{n_k} X_{k, i} - n_kR \notag \\
          &= (X_{k,1} - R) + (X_{k,2} - R) + \dotsb + (X_{k, n_k} - R). \label{eq:RSi}
\end{align}
When parameter $K$ in \eqref{eq:TOW_Si} is interpreted as the aspiration level $R$ in equation \eqref{eq:RSi}, $RS$ 
has the same mathematical form as a part of the recent TOW dynamics models under certain conditions. 
Hence, the regret calculation of TOW can be applied to $RS$ as well, and the regret of $RS$ also is upper bounded like TOW. 
In this work, we relaxed the assumption of equal variance in the proof for TOW. 

However, there exist certain differences between $RS$ and TOW.
The primary difference is that they model totally different phenomena.
$RS$ is modeled on how humans make decisions (satisficing), while taking into account the associated risks. 
Moreover, as explained in Supplementary information A, $RS$ is also a generalized model of $S0$ model in causal reasoning.
On the other hand, TOW is derived from physical laws like volume conservation. 
An advantage of $RS$ over TOW lies in its simplicity, clarity, and generalizability. 
As regards clarity, $RS$ is the product of ``reliability of obtained information'' and ``degree of satisficing,'' and the parameter $R$ is associated to ``aspiration.'' On the other hand, the interpretation of the parameter $K$ of TOW, which corresponds to $R$ in $RS$, is not necessarily clear.
Therefore, through straightforward generalization of these two constituent concepts, $RS$ need to be applied not only to the $K$-armed bandit problems (instead of two-armed) but also generally to reinforcement learning settings~\cite{Takahashi2016}.

\subsection*{C. Propositions \ref{prop:guarantee} and \ref{prop:regret} When the Aspiration Level $R$ is Variable}

Both of Propositions \ref{prop:guarantee} and \ref{prop:regret} assume that the aspiration level $R$ is constant.
When $R$ is variable or stochastic, similar propositions can be established just by slightly modifying the previous proofs assuming that $R$ is within a certain range. 
We show only the changes made in Proposition \ref{prop:guarantee_revise} from Proposition \ref{prop:guarantee}, and in Proposition \ref{prop:regret_revise} from \ref{prop:regret}.
In the proofs below, 
the symbols are the same as Propositions \ref{prop:guarantee} and \ref{prop:regret} except for the ones specified below.
Let the minimum and the maximum of the variable aspiration level $R$ be $R_{\min}$ and $R_{\max}$, respectively.

% We assume that the both of $A_U$ and $A_L$ are invariant even if the aspiration level $R$ changes temporally or stochastically.
% More specifically, we let the minimum and the maximum of the variable aspiration level $R$ be $R_{\min}$ and $R_{\max}$, respectively, and let the maximum of the reward probabilities in $A_L$ be $p_l$ and the minimum of the reward probabilities in $A_U$ be $p_u$, respectively. 
% Then, we assume that 
% $p_l <R_{\min} \leq R \leq R_{\max} < p_u$ 
% holds.

\begin{prop}[Modified Proposition \ref{prop:guarantee} for  Variable Aspiration Level $R$]
\label{prop:guarantee_revise}
We assume that the both of $A_U$ and $A_L$ are invariant even if the aspiration level $R$ changes temporally or stochastically.
More specifically, we assume that $p_l <R_{\min} \leq R \leq R_{\max} < p_u$ holds, where  $p_l$ and $p_u$ are the maximum of the reward probabilities in $A_L$ and the minimum of the reward probabilities in $A_U$, respectively.
Under this assumption, Proposition~\ref{prop:guarantee} is established as it is.

\end{prop}

\begin{proof}
The proof of Claim \ref{claim:lower} for Proposition \ref{prop:guarantee} needs to be changed as follows in the part where the law of large numbers is used. 
For any positive number $\epsilon$, there exists some $S$ such that we have $P \bigl( |E(a_i, s) - p_i| < (R_{\min} - p_i) / 2 \bigr) > 1 - \epsilon$ for any integer $s \in N_i$ greater than $S$.
Now, if $|E(a_i, s) - p_i| < (R_{\min} - p_i) / 2 $, we have
$RS(a_i,s) = n_i(s) \cdot \bigl(E(a_i, s) - R \bigr)
                 < n_i(s) \cdot \bigl( p_i + (R_{\min} - p_i)/2 - R_{\min} \bigr)
                 = n_i(s) \cdot (p_i - R_{\min})/2 < 0.$
% \begin{align}
% RS(a_i,s) &= n_i(s) \cdot \bigl(E(a_i, s) - R \bigr) \notag\\
%                  &< n_i(s) \cdot \bigl( p_i + \frac{R_{\min} - p_i}{2} - R_{\min} \bigr) \notag \\
%                  &= n_i(s) \cdot \frac{p_i - R_{\min}}{2} < 0.
% \end{align}
Hereafter the proof is the same as that of Claim \ref{claim:lower} in Proposition \ref{prop:guarantee}.
The proof of Claim \ref{claim:upper} needs no change. 

The proof of Proposition \ref{prop:guarantee} needs to be similarly changed as follows. % in the part where the law of large numbers is used.
For any positive number $\epsilon$, there exists some $S$ such that we have $P \bigl( |E(a_k, s) - p_k| <  (p_k  - R_{\max}) / 2 \bigr) > 1 - \epsilon$ for any integer $s \in N_i$ greater than $S$.
Now if $|E(a_k, s) - p_k| <  (p_k  - R_{\max}) / 2$, we have
$RS(a_k, s) = n_k(s) \cdot \bigl(E(a_k, s) - R \bigr)
     > n_k(s) \cdot \bigl( p_k + (R_{\max} - p_k)/2 - R_{\max} \bigr)
     = n_k(s) \cdot (p_k - R_{\max})/2 > 0.$ 
% \begin{align}
% RS(a_k, s) &= n_k(s) \cdot \bigl(E(a_k, s) - R \bigr) \notag\\
%                  &> n_k(s) \cdot \bigl( p_k + \frac{R_{\max} - p_k}{2} - R_{\max} \bigr) \notag \\
%                  &= n_k(s) \cdot \frac{p_k - R_{\max}}{2} > 0.
% \end{align}
Hereafter the proof is the same as that of Proposition \ref{prop:guarantee}.
\end{proof}

\begin{prop}[Modified Proposition \ref{prop:regret} for the Variable Aspiration Level $R$]
\label{prop:regret_revise}
% The symbols are the same as Proposition~\ref{prop:regret} except for the ones specified below.
We assume that the aspiration level $R$ satisfies $p_2 < R_{\min} \leq R \leq R_{\max} <p_1$, even if the aspiration level $R$ changes temporally or stochastically.
Under this assumption, we can still prove that the regret is upper bounded by a finite value.
\end{prop}

\begin{proof}
% Let the aspiration level $R$ be variable that satisfies $p_2 < R < p_1$.
% Let the minimum and the maximum of the variable aspiration level $R$ be $R_{\min}$ and $R_{\max}$, respectively.
% That is, $p_2 < R_{\min} \leq R \leq R_{\max} <p_1$.
Let $RS(a_i, s) = n_i(s) \cdot \bigl(E(a_i, s) - R \bigr) \,\,\, (i = 1, 2, \dotsc, K)$.
Here, we define $RS_{bd}(a_1, s)$ as follows:
$RS_{bd}(a_1, s) = n_1(s) \cdot \bigl(E(a_1, s) - R_{\max}\bigr)  
         \leq n_1(s) \cdot \bigl(E(a_1, s) - R\bigr)  
         = RS(a_1, s).$
% \begin{align}
% RS_{bd}(a_1, s) &= n_1(s) \cdot \bigl(E(a_1, s) - R_{\max}\bigr)  \notag \\
%                      &\leq n_1(s) \cdot \bigl(E(a_1, s) - R\bigr)  \notag \\
%                      &= RS(a_1, s).
% \end{align}
Also, we define $RS_{bd}(a_i,s)$ for $i \neq 1$, as follows:
$RS_{bd}(a_i, s) = n_i(s)  \cdot \bigl(E(a_i, s) - R_{\min} \bigr) 
        \geq n_i(s)\cdot \bigl(E(a_i, s) - R\bigr) 
        = RS(a_i,s).$
% \begin{align}
% RS_{bd}(a_i, s) &= n_i(s)  \cdot \bigl(E(a_i, s) - R_{\min} \bigr) \notag \\
%                     &\geq n_i(s)\cdot \bigl(E(a_i, s) - R\bigr) \notag \\
%                     &= RS(a_i,s).
% \end{align}
The suffix $bd$ means using the boundary of the aspiration level $R$.
If we let $R_i = R_{\max} \,\,\, (i = 1), R_{\min} \,\,\, (i \neq1)$, then, we have $RS_{bd}(a_i, s) = n_i(s)  \cdot \bigl(E(a_i, s) - R_i \bigr) \,\,\, (i = 1, 2, \dotsc, K)$.

The expectation $E$ and the variance $V$ of $RS_{bd}(a_i, s)$ are $E[RS_{bd}(a_i, s)] = n_i(s) (p_i - R_i)$ and $V[RS_{bd}(a_i, s)] = n_i(s) \sigma_i^2$, respectively, where $\sigma_i^2 = p_i(1 - p_i)$.
Let $\Delta RS_ i(s) = RS(a_1, s) - RS(a_i, s) \,\,\, (i \neq1)$, and $\Delta RS_{bd, i}(s) = RS_{bd}(a_1, s) - RS_{bd}(a_i, s) \,\,\, (i \neq~ 1)$.
Note that $\Delta RS_i(s) \geq \Delta RS_{bd, i}(s)$ holds because $RS(a_1, s) \geq RS_{bd}(a_1, s)$ and $RS(a_i,s) \leq RS_{bd}(a_i,s) \,\,\, (i \neq 1)$.

$E[\Delta RS _{bd, i}(s)]$, which is the expectation of $\Delta RS_{bd, i}(s)$, is evaluated as follows:
\begin{align}
E[\Delta RS_{bd, i}(s)] &= n_1(s)(p_1 - R_{\max}) - n_i(s)(p_i - R_{\min}) \notag \\
                                                &= n_1(s)(p_1 - R_{\max}) + n_i(s)(R_{\min} - p_i) \notag \\
                                                &\geq n_1(s)(p_1 - R_{\max}) + n_i(s)(R_{\min} - p_2) \,\,\,
                                                (\because i \neq 1, \,\,\, p_2 \geq p_i)  \notag \\
                                                & \geq (n_1(s) + n_i(s)) \min \bigl((p_1 - R_{\max}), (R_{\min} - p_2) \bigr) \,\,\,(\because p_1 - R_{\max} > 0, R_{\min} - p_2 > 0).
\end{align}
By Proposition \ref{prop:guarantee_revise}, if the step number $s$ is sufficiently large, then $n_1(s) \rightarrow s$ with probability 1 (the same approximation as in the proof of Proposition \ref{prop:regret}, hence the same note applies).
Hence, $E[\Delta RS_{bd, i}(s)] \geq s \cdot~ \min \bigl((p_1 - R_{\max}), (R_{\min} - p_2)\bigr)$.
Also, $V[\Delta RS_{bd, i}(s)]$ of the variance of $\Delta RS_{bd, i}(s)$ is evaluated as follows: $V[\Delta RS_{bd, i}(s)] \leq (n_1(s) + n_i(s)) \sigma_{1, i}^2 \leq s \sigma_{1, i}^2$, where $\sigma_{1,i} = \max (\sigma_1, \sigma_i)$. 

% \noindent (Note: As well as the proof of Proposition \ref{prop:regret}, the approximate calculations are performed  in the proof.
% Hence, there is a possibility that the value of the upper bound of the regret is not accurate due to the approximation error.)

By the central limit theorem, $\Delta RS_{bd, i}(s)$ follows the normal distribution with expectation $E[\Delta RS_{bd, i}(s)]$ and variance $V[\Delta RS_{bd, i}(s)]$.
The probability that $\Delta RS_{bd, i}(s) \leq 0$ is $Q(E[\Delta RS_{bd, i}(s)] / \sqrt{V[\Delta RS_{bd, i}(s)]})$.
Then, $P[s = n+1; I = i]$, which the probability that action $a_i$ is chosen in the $(n+1)$-th step, is given by
\begin{align}
P[s = n + 1, I = i] &\leq P[RS(a_j, n) \leq RS(a_i, n) \,\,\, (\forall j \neq i)] \notag \\
                               &\leq P[\Delta RS_i(n) \leq 0] \notag \\
                               &\leq P[\Delta RS_{bd, i}(n) \leq 0] \,\,\, (\because \Delta RS_i(s) \geq \Delta RS_{bd, i}(s)) \notag \\
                               &= Q\biggl( \frac{ E[\Delta RS_{bd, i}(n)] }{ \sqrt{V[\Delta RS_{bd, i}(n)]} } \biggr) \notag \\
                               &\leq Q \biggl( \frac{ \sqrt{n} \cdot \min \bigl((p_1 - R_{\max}), (R_{\min} - p_2) \bigr)}{ \sigma_{1, i} } \biggr)  \notag\\
                               &= Q( \phi_i \sqrt{n}), \label{eq:regret_phi}
\end{align}
where we set $\phi_i = \min \bigl((p_1 - R_{\max}), (R_{\min}- p_2)\bigr) / (\sigma_{1,i})$.
Hereafter the proof is the same as that of Proposition \ref{prop:regret}.
As a result, the upper bound of regret is obtained by replacing $\phi_i$ in Eq. \eqref{eq:limit_regret} with $\phi_i$ set in Eq. \eqref{eq:regret_phi}.
\end{proof}

\end{document}